\documentclass[12pt]{article}

\usepackage{natbib}
\usepackage{epsfig}
\usepackage{natbib}
\usepackage{times}
\usepackage{amsmath,amsthm,amsfonts,amssymb,mathrsfs}
\usepackage[fleqn,tbtags]{mathtools}

\newcommand{\bproof}{\begin{proof}}
\newcommand{\eproof}{\end{proof}}

\begin{document}
\bibliographystyle{plainnat}
\makeatletter

\renewcommand{\theequation}{\thesection.\arabic{equation}}
\numberwithin{equation}{section}
\renewcommand{\bar}{\overline}
\newtheorem{theorem}{Theorem}[section]
\newtheorem{question}{Question}[section]
\newtheorem{proposition}{Proposition}[section]
\newtheorem{lemma}{Lemma}[section]
\newtheorem{corollary}{Corollary}[section]
\newtheorem{definition}{Definition}[section]
\newtheorem{problem}{\em Problem}[section]
\newtheorem{remark}{Remark}[section]
\newtheorem{example}{Example}[section]
\newtheorem{case}{Case}[section]
\newtheorem{assumption}{Assumption}[section]

\renewcommand\proofname{\bf Proof} 
\def\eop{$\rule{1.3ex}{1.3ex}$}
\renewcommand\qedsymbol\eop  
\numberwithin{equation}{section}
\makeatletter

\newcommand{\XXX}{{\bf XXX~}}
\newcommand{\beq}{\begin{equation}} \newcommand{\eeq}{\end{equation}}
\newcommand{\bz}{{\bf z}}
\newcommand{\bx}{{\bf x}}
\newcommand{\bt}{{\bf t}}
\newcommand{\bi}{\begin{itemize}}
\newcommand{\be}{\begin{enumerate}}
\newcommand{\ei}{\end{itemize}}
\newcommand{\ee}{\end{enumerate}}
\newcommand{\calH}{{\cal H}}
\newcommand{\E}{{\mathbb E}}
\newcommand{\Om}{{\Theta}}
\newcommand{\om}{{\theta}}
\newcommand{\gam}{{\gamma}}
\newcommand{\Gam}{{\Gamma}}
\newcommand{\homega}{{\hat \om}}
\newcommand{\bS}{{\mathbb S}}
\newcommand{\R}{{\mathbb R}}
\newcommand{\N}{{\mathbb N}}
\newcommand{\calE}{{\cal E}}
\newcommand{\calG}{{\cal G}}
\newcommand{\calV}{{\cal V}}
\newcommand{\calK}{{\cal K}}
\newcommand{\calN}{{\cal N}}
\newcommand{\calU}{{\cal U}}
\newcommand{\calT}{{\cal T}}
\newcommand{\calY}{{\cal Y}}
\newcommand{\calO}{{\cal O}}
\newcommand{\calX}{{\cal X}}
\newcommand{\calW}{{\cal W}}
\newcommand{\W}{{\cal W}}
\newcommand{\G}{{\cal G}}
\newcommand{\K}{{\cal K}}
\newcommand{\OO}{{\bf O}}
\newcommand{\hK}{{\hat K}}
\newcommand{\X}{{\cal X}}
\newcommand{\M}{{\cal M}}
\newcommand{\KG}{{\cal K(\calG)}}
\newcommand{\lam}{{\lambda}}
\newcommand{\calM}{{\cal M}}
\newcommand{\calA}{{\cal A}}
\newcommand{\calB}{{\cal B}}
\newcommand{\calL}{{\cal L}}
\newcommand{\calD}{{\cal D}}
\newcommand{\calR}{{\cal R}}
\newcommand{\pp}{{\cal P}}
\newcommand{\hc}{{\hat c}}
\newcommand{\ck}{{c_K}}
\newcommand{\hL}{{\hat L}}
\newcommand{\tL}{{\bar L}}
\newcommand{\sK}{{\SSS K}}
\newcommand{\hg}{{g_K}}
\newcommand{\tf}{{f_K}}
\newcommand{\hy}{{y_K}}
\newcommand{\haty}{{\hat y}}
\newcommand{\hG}{{\hat \Gam}}
\newcommand{\vt}{{\vec t}}
\newcommand{\vv}{{\vec v}}
\newcommand{\lb}{{\langle}}
\newcommand{\rb}{{\rangle}}
\newcommand{\by}{{\bf y}}
\newcommand{\btau}{{\bf \tau}}
\newcommand{\bu}{{\bf u}}
\newcommand{\bv}{{\bf v}}
\newcommand{\tby}{\tilde{{\bf y}}}
\newcommand{\Sb}{{\bf S}}
\newcommand{\Mb}{{\bf M}}
\newcommand{\Ob}{{\bf O}}
\newcommand{\SSS}{\scriptscriptstyle}
\def\boldf#1{\hbox{\rlap{$#1$}\kern.4pt{$#1$}}}
\newcommand{\balpha}{{\boldf \alpha}}
\newcommand{\wh}{\hat w}
\newcommand{\Wh}{\hat W}
\newcommand{\wb}{\bar w}
\newcommand{\Wb}{\bar W}
\newcommand{\xb}{\bar x}
\newcommand{\cb}{\bar c}
\newcommand{\trans}{^{\scriptscriptstyle \top}}
\newcommand{\tW}{\tilde{W}}
\newcommand{\tw}{\tilde{w}}
\newcommand{\hbeta}{{\hat \beta}}

\newcommand{\figsheight}{4.0cm}
\renewcommand\baselinestretch{1}

\mathtoolsset{showonlyrefs,showmanualtags}

\begin{titlepage}
\advance\topmargin by 0.5in
\begin{center}

\vspace{1.5truecm} {\Large Taking Advantage of Sparsity in Multi-Task Learning}

\vspace{1.2truecm}

\end{center}

\begin{center}

{\bf Karim Lounici}$^{(1)}$
\\ \vspace{.2truecm}
{\bf Massimiliano Pontil}$^{(1)}$
\\ \vspace{.2truecm}
{\bf Alexandre B. Tsybakov}$^{(1)}$ 
\\ \vspace{.2truecm}
{\bf Sara van de Geer}$^{(3)}$

\vspace{1.0truecm}

\noindent (1) LPMA and CREST \\
3, Av. Pierre Larousse, \\
92240 Malakoff, France \\
{\em \{karim.lounici,alexandre.tsybakov\}@ensae.fr} \\ 

\vspace{.75truecm}

\noindent (2) Department of Computer Science \\
University College London \\
Gower Street, London WC1E, England, UK \\
E-mail: {\em m.pontil@cs.ucl.ac.uk}

\vspace{.75truecm}

\noindent (3) Seminar f\"ur Statistik, ETH Z\"urich \\
LEO D11, 8092 Z\"urich \\
Switzerland \\
{\em geer@stat.math.ethz.ch}

\vspace{.75truecm}

\begin{center}
February 13, 2009
\end{center}

\end{center}

\vspace{.5truecm}

\begin{abstract}
{\noindent We study the problem of estimating multiple linear
regression equations for the purpose of both prediction and variable
selection. Following recent work on multi-task learning
\cite{AEP}, we assume that the regression vectors share the same
sparsity pattern. This means that the set of relevant predictor
variables is the same across the different equations. This assumption
leads us to consider the Group Lasso as a candidate estimation
method. We show that this estimator enjoys nice sparsity oracle
inequalities and variable selection properties. The results hold 
under a certain restricted eigenvalue condition and a coherence
condition on the design matrix, which naturally extend recent work in
\cite{BRT,lounici2008snc}. In particular, in 
the multi-task learning scenario, in which the number of tasks can
grow, we are able to remove completely the effect of the number of
predictor variables in the bounds. Finally, we show how our results can be extended to more general noise
distributions, of which we only require the variance to be finite.
}
\end{abstract}
\end{titlepage}

\bibliographystyle{plain}
\renewcommand{\theequation}{\thesection.\arabic{equation}}
\numberwithin{equation}{section}
\renewcommand{\bar}{\overline}

\section{Introduction}
We study the problem of estimating multiple regression equations under
sparsity assumptions on the underlying regression coefficients. More precisely,
we consider multiple Gaussian regression models,
\begin{equation}
\begin{array}{lcl}
y_1 & = & X_1 \beta^*_1 + W_1 \\
y_2 & = & X_2 \beta^*_2 + W_2 \\
~ & \vdots & ~\\
y_T & = & X_T \beta^*_T + W_T
\end{array}
\label{eq:s1}
\end{equation}
where, for each $t = 1,\dots,T$, we let $X_t$ be a prescribed $n
\times M$ design matrix, $\beta_t^*$ the unknown vector of
regression coefficients and $y_t$ an $n$-dimensional vector of
observations.  We assume that $W_1,\dots,W_T$ are {\em i.i.d.} zero
mean random vectors.

We are interested in estimation methods which work well even when the
number of parameters in each equation is much larger than the number
of observations, that is, $M\gg n$. This situation may arise in many
practical applications in which the predictor variables are inherently
high dimensional, or it may be ``costly'' to observe response
variables, due to difficult experimental procedures, see, for example
\cite{AEP} for a discussion. 

Examples in which this estimation problem is relevant range from
multi-task learning
\cite{AEP,CCG_08,maurer,obozinski2008usr} and conjoint
analysis (see, for example, \cite{EPT,lenk} and
references therein) to longitudinal data analysis
\cite{diggle2002ald} as well as the analysis of panel data
\cite{hsiao2003apd,wooldridge2002eac}, among others. In particular,
multi-task learning provides a main motivation for our study. In
that setting each regression equation corresponds to a different
learning task (the classification case can be treated similarly); in
addition to the requirement that $M \gg n$, we are also interested
in the case that the number of tasks $T$ is much larger than $n$.
Following \cite{AEP} we assume that there are only few common
important variables which are shared by the tasks. A general goal of
this paper is to study the implications of this assumption from a
statistical learning view point, in particular, to quantify the
advantage provided by the large number of tasks to learn both the
underlying vectors $\beta^*_1,\dots,\beta^*_T$ as well as to select
common variables shared by the tasks.

Our study pertains and draws substantial ideas from the recently
developed area of compressed sensing and sparse estimation (or
sparse recovery), see \cite{BRT,candes2007dss,donoho2006srs} and
references therein. A central problem studied therein is that of
estimating the parameters of a (single) Gaussian regression model.
Here, the term ``sparse'' means that most of the components of the
underlying $M$-dimensional regression vector are equal to zero. A
main motivation for sparse estimation comes from the observation
that in many practical applications $M$ is much larger than the
number $n$ of observations but the underlying model is
(approximately) sparse, see \cite{candes2007dss,donoho2006srs} and
references therein. Under this circumstance ordinary least squares
will not work. A more appropriate method for sparse estimation is
the $\ell_1$-norm penalized least squares method, which is commonly
referred to as the Lasso method. In fact, it has been recently shown
by different authors, under different conditions on the design
matrix, that the Lasso satisfies sparsity oracle inequalities, see
\cite{BRT,BTWAnnals07,bunea2007soi,vandegeer2008hdg} and references
therein. Closest to our study in this paper is \cite{BRT}, which
relies upon a Restricted Eigenvalue (RE) assumption. The results of
these works make it possible to estimate the parameter $\beta$ even
in the so-called ``{\em p much larger than n}" regime (in our
notation, the number of predictor variables $p$ corresponds to
$MT$).

In this paper, we assume that the vectors
$\beta_1^*,\dots,\beta^*_T$ are not only sparse but also have the same
sparsity pattern. This means that the set of indices which correspond
to non zero components of $\beta_t^*$ is the same for every
$t=1,\dots,T$. In other words, the response variable associated with
each equation in \eqref{eq:s1} depends only on a small subset (of size
$s\ll M$) of the corresponding predictor variables and the set of
relevant predictors is preserved across the different equations. This
assumption, that we further refer to as {\it structured sparsity
assumption}, is motivated by some recent work on multi-task learning
\cite{AEP}. It naturally leads to an extension of the Lasso
method, the so-called group Lasso
\cite{group_lasso}, in which the error term is the average residual
error across the different equations and the penalty term is a mixed
$(2,1)$-norm. The structured sparsity assumption induces a
relation between the responses and, as we shall see, can be used to
improve estimation.

The paper is organized as follows. In Section \ref{sec:2} we define
the estimation method and comment on previous related work. In
Section \ref{sec:3} we study the oracle properties of this estimator
when the errors $W_t$ are Gaussian. Our main results concern upper
bounds on the prediction error and the distance between the
estimator and the true regression vector $\beta^*$. Specifically,
Theorem \ref{th1} establishes that under the above structured
sparsity assumption on $\beta^*$, the prediction error is
essentially of the order of $s/n$. In particular, in the multi-task
learning scenario, in which $T$ can grow, we are able to remove
completely the effect of the number of predictor variables in the
bounds. Next, in Section \ref{sec:4}, under a stronger condition on
the design matrices, we describe a simple modification of our method
and show that it selects the correct sparsity pattern with an
overwhelming probability (Theorem \ref{est-supnorm}). We also find
the rates of convergence of the estimators for mixed $(2,1)$-norms
with $1\le p\le \infty$ (Theorem \ref{Signnonasymp}). The techniques
of proofs build upon and extend those of \cite{BRT} and
\cite{lounici2008snc}. Finally, in Section \ref{sec:5} we
discuss how our results can be extended to more general noise
distributions, of which we only require the variance to be finite.

\section{Method and related work}
\label{sec:2}

In this section we first introduce some notation and then describe
the estimation method which we analyze in the paper. As stated
above, our goal is to estimate $T$ linear regression functions
identified by the parameters $\beta_1^*,\dots,\beta_T^* \in \R^M$.
We may write the model \eqref{eq:s1} in compact notation as \beq y =
X \beta^{*} + W \label{eq:model} \eeq where $y$ and $W$ are the
$nT$-dimensional random vectors formed by stacking the vectors
$y_1,\dots,y_T$ and the vectors $W_1,\dots,W_T$, respectively.
Likewise $\beta^*$ denotes the vector obtained by stacking the
regression parameter vectors $\beta^*_1,\dots,\beta^*_T$. Unless
otherwise specified, all vectors are meant to be column vectors.
Thus, for every $t \in \N_T$, we write $y_t=(y_{ti}: i \in
\N_n)^\top$ and $W_t=(W_{ti}: i \in \N_n)^\top$, where, hereafter,
for every positive integer $k$, we let $\N_k$ be the set of integers
from $1$ and up to $k$. The $nT \times MT$ block diagonal design
matrix $X$ has its $t$-th block formed by the $n \times M$ matrix
$X_t$. We let $x_{t1}^\top,\dots,x_{tn}^\top$ be the row vectors
forming $X_t$ and $(x_{ti})_j$ the $j$-th component of the vector
$x_{ti}$. Throughout the paper we assume that $x_{ti}$ are
deterministic.

For every $\beta \in \R^{MT}$ we introduce $(\beta)^j \equiv \beta^j
= (\beta_{tj}: t \in \N_T)^\top$, that is, the vector formed by the
coefficients corresponding to the $j$-th variable. For every $1\le p
<\infty$ we define the mixed $(2,p)$-norm of $\beta$ as
$$\|\beta\|_{2,p} =
\left(\sum_{j=1}^M \left(
\sum_{t=1}^T \beta_{tj}^2\right)^\frac{p}{2} 
\right)^\frac{1}{p} = \left(\sum_{j=1}^M
\|\beta^j\|^p\right)^\frac{1}{p} 
$$
and the $(2,\infty)$-norm of $\beta$ as
$$
\|\beta\|_{2,\infty} = \max_{1 \leq j \leq M} \|\beta^j\|,
$$
where $\|\cdot\|$ is the standard Euclidean norm.

If $J \subseteq \N_M$ we let $\beta_J \in \R^{MT}$ be the vector
formed by stacking the vectors $(\beta^j I\{j \in J\} : j \in \N_M)$,
where $I\{\cdot\}$ denotes the indicator function. Finally we set
$J(\beta) = \{j:\beta^j \neq 0,~j \in \N_M\}$ and $M(\beta) =
|J(\beta)|$ where $|J|$ denotes the cardinality of set $J\subset
\{1,\dots,M\}$. The set $J(\beta)$ contains the indices of the relevant 
variables shared by the vectors $\beta_1,\dots,\beta_T$ and the number
$M(\beta)$ quantifies the level of structured sparsity across those
vectors.

We have now accumulated the sufficient information to introduce the
estimation method. We define the empirical residual error $$ {\hat
S}(\beta)= \frac{1}{nT} \sum_{t=1}^T \sum_{i=1}^n (x_{ti}\trans
\beta_t - y_{ti})^2 = \frac{1}{nT} \|X \beta - y\|^2 $$ and, for every
$\lambda > 0$, we let our estimator $\hbeta$ be a solution of the
optimization problem \cite{AEP} \beq \min_\beta {\hat S}(\beta) +
2
\lambda \|\beta\|_{2,1}. \label{eq:opt} \eeq


In order to study the statistical properties of this estimator, it
is useful to derive the optimality condition for a solution of the
problem \eqref{eq:opt}. Since the objective function in
\eqref{eq:opt} is convex, $\hbeta$ is a solution of \eqref{eq:opt}
if and only if $0$ (the $MT$-dimensional zero vector) belongs to the
subdifferential of the objective function. In turn, this condition
is equivalent to the requirement that $$ -\nabla {\hat S}(\hbeta)
\in 2 \lambda
\partial \left(\sum_{j=1}^M \|\hbeta^{j}\| \right),  $$
where $\partial$ denotes the subdifferential (see, for example,
\cite{BorLew} for more information on convex analysis). Note that
\begin{align*}
\partial \left(\sum_{j=1}^M \|\beta^{j}\| \right) =
\bigg\{ \theta \in \R^{MT}: \theta^j =
\frac{\beta^j}{\|\beta^j\|}~{\rm if}~ \beta^j \neq 0,
\\
~
\|\theta^j\| \leq 1,~{\rm if}~ \beta^j = 0,~j \in \N_M
\bigg\}.
\end{align*}
Thus, $\hbeta$ is a solution of \eqref{eq:opt} if and only if
\begin{eqnarray}
\label{eq:sol-n0}
\hspace{-.3truecm}\frac{1}{nT} (X\trans (y - X \hbeta))^j &
\hspace{-.3truecm}= \hspace{-.3truecm} &
\lambda \frac{\hbeta^j}{\|\hbeta^j\|},~~~~\mbox{if}~\hbeta^j \neq 0 \\
\label{eq:sol-n2}
 \hspace{-.3truecm}\frac{1}{nT} \|(X\trans (y - X
\hbeta))^j\| & \hspace{-.3truecm} \leq \hspace{-.3truecm} &
\lambda,~~~~~~~~~~~~~\mbox{if}~ \hbeta^j = 0.
\end{eqnarray}

Finally, let us comment on previous related work. Our estimator is a
special case of the group Lasso estimator \cite{group_lasso}.  Several
papers analyzing statistical properties of the group Lasso appeared
quite recently
\cite{bach07,ChesHeb07,Horowitz08,koltch_y08,MeierGeerBuhlm06,MGB08,NarRin08,SPAM}. Most
of them are focused on the group Lasso for additive models
\cite{Horowitz08,koltch_y08,MGB08,SPAM} or generalized linear models
\cite{MeierGeerBuhlm06}. Special choice of groups is studied in
\cite{ChesHeb07}. Discussion of the group Lasso in a relatively
general setting is given by Bach \cite{bach07} and Nardi and Rinaldo
\cite{NarRin08}. Bach \cite{bach07} assumes that the predictors
$x_{ti}$ are random with a positive definite covariance matrix and
proves results on consistent selection of sparsity pattern
$J(\beta^*)$ when the dimension of the model ($p=MT$ in our case) is
fixed and $n\to\infty$. Nardi and Rinaldo \cite{NarRin08} consider a
setting that covers ours and address the issue of sparsity oracle
inequalities in the spirit of \cite{BRT}. However, their bounds are
too coarse (see comments in Section \ref{sec:3} below). Obozinski et
al.~\cite{obozinski2008usr} replace in (\ref{eq:opt}) the
$(2,1)$-norms by $(q,1)$-norms with $q>1$ and show that the
resulting estimator achieves consistent selection of the sparsity
pattern under the assumption that all the rows of matrices $X_t$ are
independent Gaussian random vectors with the same covariance matrix.

This literature does not demonstrate theoretical advantages of the
group Lasso as compared to the usual Lasso. One of the aims of this
paper is to show that such advantages do exist in the multi-task
learning setup. In particular, our Theorem \ref{th1} implies that
the prediction bound for the group Lasso estimator that we use here
is by at least a factor of $T$ better than for the standard Lasso
under the same assumptions. Furthermore, we demonstrate that as the
number of tasks $T$ increases the dependence of the bound on $M$
disappears, provided that $M$ grows at the rate slower than
$\exp({\sqrt{T}})$.

\section{Sparsity oracle inequality}
\label{sec:3}

Let $1\le s \le M$ be an integer that gives an upper bound on the
structured sparsity $M(\beta^*)$ of the true regression vector
$\beta^*$. We make the following assumption. 
\begin{assumption}
\label{RE}
There exists a positive number $\kappa=\kappa(s)$ such that
\begin{align*}
\min \bigg\{ \frac{\sqrt{\Delta\trans X\trans X
\Delta}}{\sqrt{n}\|\Delta_J\|}~:~ |J| \leq s, \Delta\in
\R^{MT}\setminus \{0\}, \\
 \,\|\Delta_{J^c}\|_{2,1} \leq 3 \|\Delta_J\|_{2,1} \bigg\}
\geq \kappa, \end{align*}\label{ass} where $J^c$ denotes the
complement of the set of indices~$J$.
\end{assumption}

To emphasize the dependency of Assumption \ref{RE} on $s$, we will
sometimes refer to it as Assumption RE($s$). This is a natural
extension to our setting of the Restricted Eigenvalue assumption
for the usual Lasso and Dantzig selector from \cite{BRT}. The $\ell_1$
norms are now replaced by the mixed (2,1)-norms. Note that, however,
the analogy is not complete. In fact, the sample size $n$ in the usual
Lasso setting corresponds to $nT$ in our case, whereas in Assumption
\ref{RE} we consider $\sqrt{\Delta\trans X\trans X
\Delta/n}$ and not $\sqrt{\Delta\trans X\trans X \Delta/(nT)}$. This
is done in order to have a correct normalization of $\kappa$ without
compulsory dependence on $T$ (if we use the term $\sqrt{\Delta\trans
X\trans X \Delta/(nT)}$ in Assumption \ref{RE}, then $\kappa\sim
T^{-1/2}$ even in the case of the identity matrix $X\trans X/n$).

Several simple sufficient conditions for Assumption \ref{RE} with
$T=1$ are given in \cite{BRT}. Similar sufficient conditions can be
stated in our more general setting. For example, it is enough to
suppose that each of the matrices $X_t\trans X_t/n$ is positive
definite or satisfies a Restricted Isometry condition as in
\cite{candes2007dss} or the coherence condition (cf. Lemma
\ref{lem:2} below).

\begin{lemma}\label{lem:1}
Consider the model \eqref{eq:s1} for $M \geq 2$ and $T,n \geq 1$.
Assume that the random vectors $W_1,\dots,W_T$ are i.i.d. Gaussian
with zero mean and covariance matrix $\sigma^2 I_{n \times n}$, all
diagonal elements of the matrix $X\trans X/n$ are equal to $1$ and
$M(\beta^*) \leq s$. Let
$$
\lambda = \frac{2\sigma}{\sqrt{nT}}\left(1 + \frac{A\log
M}{\sqrt{T}}\right)^{1/2},
$$
where $A > 8$ and let $q = \min(8\log M, A\sqrt{T}/8)$. Then with probability at least $1 - M^{1-q}$, for any solution $\hbeta$ of problem
\eqref{eq:opt} and all $\beta \in \R^{MT}$ we have
\begin{align}
\label{eq:1}
&\frac{1}{nT} \|X (\hbeta - \beta^*)\|^2 +  \lambda \|\hbeta- \beta\|_{2,1} \leq& \\
\nonumber
&~~~~~~~~ \leq \frac{1}{nT} \|X (\beta - \beta^*)\|^2 +  4 \lambda \sum_{j\in J(\beta)} \|\hbeta^{j}-\beta^j\|, & \\
\label{eq:2}
&\frac{1}{nT} \max_{1 \leq j \leq M} \|(X\trans X (\beta^* -\hbeta))^j\|
\leq  \frac{3}{2} \lambda,~~~~~& \\
\label{eq:3} &M(\hbeta) \leq  \frac{4 \phi_{\rm max}}{\lambda^2
nT^2} \|X (\hbeta - \beta^*)\|^2,~~~~~~~~~~&
\end{align}
where $\phi_{\rm max}$ is the maximum eigenvalue of the matrix
$ X\trans X/n$.
\end{lemma}

\begin{proof}
For all $\beta \in \R^{MT}$, we have
$$
\frac{1}{nT} \|X \hbeta - y\|^2 +  2\lambda \sum_{j=1}^M \|\hbeta^j\| \leq
\frac{1}{nT} \|X \beta - y\|^2 +  2 \lambda \sum_{j=1}^M \|\beta^j\|
$$
which, using $y=X\beta^*+W$, is equivalent to
\begin{align}
\frac{1}{nT} &\|X (\hbeta - \beta^*)\|^2 \leq \frac{1}{nT} \|X
(\beta -
\beta^*)\|^2  \nonumber \\
&+ \frac{2}{nT} W\trans X (\hbeta-\beta)+ 2 \lambda
\sum_{j=1}^M\big( \|\beta^j\| - \|\hbeta^j\|\big). \label{eq:a1}
\end{align} By H\"older's inequality, we have that
$$
W\trans X (\hbeta-\beta) \leq \|X\trans W\|_{2,\infty} \|\hbeta - \beta\|_{2,1}
$$
where 
$$
 \|X\trans W\|_{2,\infty} =  \max_{1 \leq j \leq M} \sqrt{\sum_{t=1}^T
\left(\sum_{i=1}^n (x_{ti})_j W_{ti} \right)^2}.
$$
Consider the random event
$$ {\cal A} = \left\{\frac1{nT}\|X\trans W\|_{2,\infty} \leq
\frac{\lambda}{2} \right\}.
$$
Since we assume all diagonal elements of the matrix $X\trans X/n$ to be equal
to $1$, the random variables $$V_{tj} =
\frac1{\sigma\sqrt{n}}\sum_{i=1}^n (x_{ti})_j W_{ti},$$
$t=1,\dots,T$, are {\em i.i.d.} standard Gaussian. Using this fact we can
write, for any $j=1,\dots, M$,
\begin{eqnarray*}
&&{\rm Pr}\left(\sum_{t=1}^T \left(\sum_{i=1}^n (x_{ti})_j W_{ti}
\right)^2 \ge \frac{\lambda^2(nT)^2}{4}\right)
\\
&&\quad = {\rm Pr}\left(\chi_T^2 \ge
\frac{\lambda^2nT^2}{4\sigma^2}\right)
\\
&&\quad = {\rm Pr}\left(\chi_T^2 \ge T + A\sqrt{T}\log M\right),
\end{eqnarray*}
where $\chi_T^2$ is a chi-square random variable with $T$ degrees of
freedom. We now apply Lemma \ref{chi}, the union bound and the fact
that $A>8$ to get
\begin{eqnarray*}
{\rm Pr}({\cal A}^c) \le  M\exp\left(
-\frac{A \log M}{8} \min\left(\sqrt{T},A\log M\right)
\right)
\\
 \le M^{1-q}.
\end{eqnarray*}
  It follows from
\eqref{eq:a1} that, on the event ${\cal A}$.
\begin{align}
\nonumber & \frac{1}{nT} \|X (\hbeta - \beta^*)\|^2 + \lambda
\sum_{j=1}^M \|\hbeta^j-\beta^j\| \leq\nonumber \\
\nonumber & \frac{1}{nT} \|X (\beta - \beta^*)\|^2 + 2 \lambda
\sum_{j=1}^M\big(
\|\hbeta^j-\beta^j\| + \|\beta^j\| - \|\hbeta^j\|\big) \\
\nonumber &  \leq \frac{1}{nT} \|X (\beta - \beta^*)\|^2 + 4 \lambda
\sum_{j \in J(\beta)} \|\hbeta^j-\beta^j\|,
\end{align}
which coincides with \eqref{eq:1}. To prove \eqref{eq:2}, we use the
inequality \beq \frac{1}{nT} \max_{1 \leq j \leq M}
\|(X\trans(y-X\hbeta))^j\| \leq
\lambda, \label{eq:nec} \eeq which follows from (\ref{eq:sol-n0})
and (\ref{eq:sol-n2}). Then,
\begin{align}
\nonumber
&\frac{1}{nT} \|(X\trans(X(\hbeta - \beta^*)))^j\| \leq \nonumber \\
&\frac{1}{nT} \|(X\trans(X\hbeta - y))^j\| + \frac{1}{nT}
\|(X\trans W)^j\|, \nonumber \end{align}
where we have used $y = X \beta^* + W$ and the
triangle inequality. The result then follows by combining the last
inequality with inequality \eqref{eq:nec} and using the definition
of the event ${\cal A}$.

Finally, we prove \eqref{eq:3}. First, observe that, on the event
${\cal A}$,
 \beq \nonumber \frac{1}{nT} \|(X\trans X(\hbeta -
\beta^*))^j\| \geq \frac{\lambda}{2},~~~{\rm if}~\hbeta^j \neq 0.
\eeq This fact follows from \eqref{eq:sol-n0}, \eqref{eq:model} and
the definition of the event ${\cal A}$. The following chain yields
the result:
\begin{eqnarray}
\nonumber
M(\hbeta) & \leq &
\frac{4}{\lambda^2 (nT)^2}
\sum_{j \in J(\hbeta)} \|(X\trans X (\hbeta - \beta^*))^j\|^2 \\
\nonumber
~ & \leq &
\frac{4}{\lambda^2 (nT)^2}  \sum_{j=1}^M \|(X\trans X (\hbeta - \beta^*))^j\|^2 \\
\nonumber ~ & = & \frac{4}{\lambda^2 (nT)^2} \|X\trans X (\hbeta -
\beta^*)\|^2 \nonumber \\
& \leq & \frac{4\phi_{\rm max}} {\lambda^2 nT^2} \|X(\hbeta-
\beta^*)\|^2.
\nonumber
\end{eqnarray}
\end{proof}


We are now ready to state the main result of this section.

\begin{theorem}\label{th1}
Consider the model \eqref{eq:s1} for $M \geq 2$ and $T,n \geq 1$.
Assume that the random vectors $W_1,\dots,W_T$ are i.i.d. Gaussian
with zero mean and covariance matrix $\sigma^2 I_{n \times n}$, all
diagonal elements of the matrix $X\trans X/n$ are equal to $1$ and
$M(\beta^*) \leq s$. Furthermore let Assumption \ref{ass} hold with
$\kappa=\kappa(s)$
 and let $\phi_{\rm max}$ be the largest eigenvalue
of the matrix $X\trans X/n$. Let
$$
\lambda = \frac{2\sigma}{\sqrt{nT}}\left(1 + \frac{A\log
M}{\sqrt{T}}\right)^{1/2},
$$
where $A>8$ and let $q=\min(8\log M,
A\sqrt{T}/8)$. Then with probability at least $1 - M^{1-q}$, for any solution $\hbeta$ of problem \eqref{eq:opt}
we have
\begin{eqnarray}
\label{eq:t1} \hspace{-.6truecm} \frac{1}{nT} \|X(\hbeta -
\beta^*)\|^2 \hspace{-.2truecm} & \leq & \hspace{-.2truecm} \frac{64
\sigma^2}{\kappa^2} \frac{s}{n}\left(1 + \frac{A\log
M}{\sqrt{T}}\right)~~\hspace{.3truecm} \\
\label{eq:t2} \frac1{\sqrt {T}}\|\hbeta - \beta^*\|_{2,1}
\hspace{-.2truecm} & \leq & \hspace{-.2truecm} \frac{32
\sigma}{\kappa^2} \frac{s}{\sqrt{n}}\sqrt{1 + \frac{A\log
M}{\sqrt{T}}}~~ \\
\label{eq:t3} \hspace{-.2truecm} M(\hbeta) \hspace{-.2truecm} & \leq
& \hspace{-.2truecm}  \frac{64 \phi_{\rm max}}{\kappa^2} s.~~
\end{eqnarray}
If, in addition, Assumption RE(2$s$) holds, then with the same
probability for any solution $\hbeta$ of problem \eqref{eq:opt} we
have
\begin{eqnarray}
\label{eq:t4} \hspace{-.2truecm}
\frac{1}{\sqrt{T}} \|\hbeta - \beta^*\| \hspace{-.2truecm} & \leq &
\hspace{-.2truecm} \frac{8 \sqrt{10} \sigma}{\kappa^2(2s)} \sqrt{\frac{s}{n}}
\sqrt{1 + \frac{A\log M}{\sqrt{T}}}. \hspace{.3truecm}
\end{eqnarray}
\end{theorem}
\begin{proof} We act similarly to the proof of Theorem 6.2 in
\cite{BRT}. Let $J = J(\beta^*) = \{j: (\beta^*)^j \neq 0\}$.
By inequality \eqref{eq:1} with $\beta = \beta^*$ we have,
on the even~${\cal A}$, that
\begin{eqnarray}
\hspace{-.2truecm}\frac{1}{nT} \|X (\hbeta -
\beta^*)\|^2\hspace{-.2truecm} &\leq& \hspace{-.2truecm}4 \lambda
\sum_{j
\in J} \|\hbeta^j - \beta^{*j}\|\nonumber\\
 \hspace{-.2truecm}&\leq&\hspace{-.2truecm} 4 \lambda \sqrt{s} \|
(\hbeta - \beta^*)_{J}\|.\label{eq:gino}
\end{eqnarray}
 Moreover by the same inequality, on the event ${\cal A}$, we have $ \sum_{j=1}^M
\|\hbeta^j - \beta^{*j}\| \leq 4  \sum_{j \in J} \|\hbeta^j -
\beta^{*j}\|$, which implies that $\sum_{j \in J^c} \|\hbeta^j -
\beta^{*j}\| \leq 3  \sum_{j \in J} \|\hbeta^j - \beta^{*j}\|$.
Thus, by Assumption \ref{ass} \beq \|(\hbeta- \beta^*)_{J}\| \leq
\frac{\| X(\hbeta -\beta^*)\|}{\kappa \sqrt{n}}. \label{eq:bb} \eeq
Now, \eqref{eq:t1} follows from \eqref{eq:gino} and \eqref{eq:bb}.
Inequality \eqref{eq:t2} follows again by noting that 
$$
\sum_{j=1}^M
\|\hbeta^j - \beta^{*j}\| \leq 4 \sum_{j \in J} \|\hbeta^j -
\beta^{*j}\| \leq 4 \sqrt{s} \|(\hbeta-\beta^*)_J\|
$$ 
and then using
\eqref{eq:t1}. Inequality \eqref{eq:t3} follows from \eqref{eq:3}
and \eqref{eq:t1}.

Finally, we prove \eqref{eq:t4}. Let $\Delta=\hbeta-\beta^*$ and let
$J'$ be the set of indices in $J^c$ corresponding to $s$ maximal in
absolute value norms $\|\Delta^{j}\|$. Consider the set
$J_{2s}=J\cup J'$. Note that $|J_{2s}|=2s$. Let
$\|\Delta_{J^c}^{(k)}\|$ denote the $k$-th largest norm in the set
$\{\|\Delta^{j}\|: \,j\in J^c\}$. Then, clearly,
$$
\|\Delta_{J^c}^{(k)}\| \le \sum_{j\in J^c}\|\Delta^{j}\|/k =
\|\Delta_{J^c}\|_{2,1}/k.
$$
This and the fact that $\|\Delta_{J^c}\|_{2,1}\le 3
\|\Delta_{J}\|_{2,1}$ on the event ${\cal A}$ implies
\begin{eqnarray*}
\sum_{j\in J_{2s}^c}\|\Delta^j\|^2 &\le& \sum_{k=s+1}^\infty
\frac{\|\Delta_{J^c}\|_{2,1}^2}{k^2} \\
&\le& \frac{\|\Delta_{J^c}\|_{2,1}^2}{s}\le
\frac{9\|\Delta_{J}\|_{2,1}^2}{s}\\
&\le& 9\sum_{j\in J}\|\Delta^j\|^2 \le 9\sum_{j\in
J_{2s}}\|\Delta^j\|^2.
\end{eqnarray*}
Therefore, on ${\cal A}$ we have
\begin{eqnarray}\label{ona}
\|\Delta\|^2 \le 10\sum_{j\in J_{2s}}\|\Delta^j\|^2\equiv
10\|\Delta_{J_{2s}}\|^2
\end{eqnarray}
and also from \eqref{eq:gino}:
\begin{eqnarray}
\hspace{-.2truecm}\frac{1}{nT} \|X \Delta\|^2
 \leq 4 \lambda \sqrt{s}
 \|\Delta_{J_{2s}}\|.\label{eq:gino1}
\end{eqnarray}
In addition, $\|\Delta_{J^c}\|_{2,1}\le 3 \|\Delta_{J}\|_{2,1}$
easily implies that
$$\|\Delta_{J_{2s}^c}\|_{2,1}\le 3 \|\Delta_{J_{2s}}\|_{2,1}.
$$
Combining these facts and
\eqref{eq:gino1} with Assumption RE(2$s$) we find that on the event
${\cal A}$ the following holds:
$$
\|\Delta_{J_{2s}}\| \le \frac{ 4 \lambda \sqrt{s}\,
T}{\kappa^2(2s)}\,.
$$
This inequality and \eqref{ona} yield \eqref{eq:t4}. \end{proof}

Theorem \ref{th1} is valid for any fixed $n,M,T$; the approach is
non-asymptotic. Some relations between these parameters are relevant
in the particular applications and various asymptotics can be
derived as corollaries. For example, in multi-task learning it is
natural to assume that $T \geq n$, and the motivation for our
approach is the strongest if also $M\gg n$. The bounds of Theorem
\ref{th1} are meaningful if the sparsity index $s$ is small as
compared to the sample size $n$ and the logarithm of the dimension
$\log M$ is not too large as compared to $\sqrt{T}$.

Note also that the values $T$ and $\sqrt{T}$ in the denominators of the
right-hand sides of (\ref{eq:t1}), (\ref{eq:t2}), and (\ref{eq:t4})
appear quite naturally. For instance, the norm $\|\hbeta -
\beta^*\|_{2,1}$ in (\ref{eq:t2}) is a sum of $M$ terms each of
which is a Euclidean norm of a vector in $\R^T$, and thus it is of the
order $\sqrt{T}$ if all the components are equal. Therefore,
(\ref{eq:t2}) can be interpreted as a correctly normalized ``error
per coefficient" bound.

Several important conclusions can be drawn from Theorem \ref{th1}.

\begin{enumerate}

\item {\it The dependence on the dimension $M$ is negligible for
large $T$.} Indeed, the bounds of Theorem \ref{th1} become
independent of $M$ if we choose the number of tasks $T$ larger than
$\log^2 M$.  A striking fact is that no relation between the sample
size $n$ and the dimension $M$ is required. This is quite in
contrast to the previous results on sparse recovery where the
assumption $ \log M=o(n)$ was considered as {\it sine qua non}
constraint. For example, Theorem \ref{th1} gives meaningful bounds
if $M=\exp({n^\gamma})$ for arbitrarily large $\gamma>0$, provided that
$T>n^{2\gamma}$. This is due to the structured sparsity assumption
that we naturally exploit in the multi-task scenario.

\item {\it Our estimator is better than the standard Lasso in
the multi-task setup.} Theorem \ref{th1} witnesses that our group
Lasso estimator admits substantially better error bounds than the
usual Lasso. Let us explain this considering the example of the
prediction error bound (\ref{eq:t4}). Indeed, for the same
multi-task setup, we can apply a usual Lasso estimator
$\hat{\beta}^{L}$, that is a solution of the following optimization
problem
\begin{equation*}
\min_{\beta}S(\beta) + 2\lambda
\sum_{t=1}^{T}\sum_{j=1}^{M}|\beta_{tj}|.
\end{equation*}
Assume, for instance, that we are in the most favorable situation
where $M<n$, each of the matrices $\frac{1}{n}X^{T}_{t}X_{t}$ is
positive definite and has minimal eigenvalue greater than $\kappa^2$
(this, of course, implies Assumption 3.1).
 We can then apply inequality
(7.8) from \cite{BRT} with 
$$
\lambda = A\sigma\sqrt{\frac{\log
(MT)}{nT}},
$$
where $A>2\sqrt{2}$, to obtain that, with probability at
least $1-(MT)^{1-\frac{A^{2}}{8}}$, it holds 
\begin{equation}\label{pre}
\frac{1}{nT}||X(\hat{\beta}^{L} - \beta^*)||^{2} \leqslant
\frac{16A^{2}}{\kappa^{2}}\sigma^{2}sT\frac{\log (MT)}{n}.
\end{equation}
Indeed, when applying (7.8) of \cite{BRT} we account for the fact
that the parameters $n$, $M$, $s$ therein correspond to $nT$, $MT$,
$sT$ in our setup, and the minimal eigenvalue of the matrix
$\frac{1}{nT}X^{T}X$ is greater than $\kappa^2/T$. Comparison with
(\ref{eq:t4}) leads to the conclusion that the prediction bound for
our estimator is by at least a factor of $T$ better than for the
standard Lasso under the same assumptions. Let us emphasize that the
improvement is due to the property that $\beta^*$ is structured
sparse. The second inherent property of our setting, that is, the fact
that the matrix ${X\trans}X$ is block-diagonal, can be characterized
as important but not indispensable. We discuss this in the next
remark.

\item {\it Theorem \ref{th1} applies to the general group Lasso
setting.} Indeed, the proofs in this section do not use the fact
that the matrix ${X\trans}X$ is block-diagonal. The only restriction
on ${X\trans}X$ is given in Assumption 3.1.  For example, Assumption
3.1 is obviously satisfied if ${X\trans}X/(nT)$ (the correctly
normalized Gram matrix of the regression model (\ref{eq:model})) has
a positive minimal eigenvalue. However, the price for having this
property (or Assumption 3.1 in general), as well as the resulting
error bounds, can be different for the block-diagonal (multi-task)
setting and the full matrix $X$ setting.

\end{enumerate}

Finally, we note that \cite{NarRin08} follow the scheme of the proof
of \cite{BRT} to derive similar in spirit to ours but coarse oracle
inequalities. Their results do not explain the advantages discussed
in the points 1--3 above. Indeed, the tuning parameter $\lambda$
chosen in \cite{NarRin08}, pp.~614--615, is larger than our $\lambda$
by at least a factor of $\sqrt{T}$. As a consequence, the
corresponding bounds in the oracle inequalities of \cite{NarRin08}
are larger than ours by positive powers of $T$.

\section{Coordinate-wise estimation and selection of sparsity pattern}
\label{sec:4}

In this section, we show how from any solution of the problem
\eqref{eq:opt} we can reliably estimate the correct sparsity pattern with
high probability.

We first introduce some more notation. We define the Gram matrix of
the design $\Psi = \frac{1}{n}X\trans X$. Note that $\Psi$ is a $MT
\times MT$ block-diagonal matrix with $T$ blocks of dimension $M
\times M$ each. We denote these blocks by
$\Psi_{t}=\frac{1}{n}X_{t}\trans X_{t}\equiv
(\Psi_{tj,tk})_{j,k=1,\dots,M}$.

In this section we assume that the following condition holds true.
\begin{assumption}\label{mutcoh}
The elements $\Psi_{tj,tk}$ of the Gram matrix $\Psi$ satisfy
\begin{equation*}
\Psi_{tj,tj}=1,\hspace{0.2cm}\forall 1\leqslant j \leqslant M,\,
1\leqslant t \leqslant T,
\end{equation*}
and
\begin{equation*}
\max_{1\leqslant t \leqslant T,j\neq k}|\Psi_{tj,tk}| \leqslant
\frac{1}{7\alpha s},
\end{equation*}
for some integer $s\geqslant 1$ and some constant $\alpha>1$.
\end{assumption}
Note that the above assumption on $\Psi$ implies Assumption \ref{ass} as
we prove in the following lemma.
\begin{lemma}\label{lem:2}
Let Assumption \ref{mutcoh} be satisfied. Then Assumption \ref{RE}
is satisfied with $\kappa  = \sqrt{1-\frac{1}{\alpha}}$.
\end{lemma}
\begin{proof}
For any subset $J$ of $\{1,\ldots,M\}$ such that $|J|\leqslant s$
and any $\Delta\in \R^{MT}$ such that
$\|\Delta_{J^{c}}\|_{2,1}\leqslant 3\|\Delta_{J}\|_{2,1}$, we have
\begin{eqnarray*}\label{stanarg}
\frac{\Delta_{J}\trans\Psi\Delta_{J}}{\|\Delta_{J}\|^{2}} &=&
1+\frac{\Delta_{J}\trans(\Psi-I_{MT\times MT})\Delta_{J}}{\|\Delta_{J}\|^{2}}\nonumber\\
&\geqslant& 1-\frac{1}{7\alpha s}\frac{\Big(\sum_{j\in
J}\sum_{t=1}^{T}|\Delta_{tj}|\Big)^{2}}
{\|\Delta_{J}\|^{2}}\\
&\geqslant&1-\frac{1}{7\alpha}
\end{eqnarray*}
where we have used Assumption \ref{mutcoh} and the Cauchy-Schwarz
inequality. Next, using consecutively Assumption \ref{mutcoh}, the
Cauchy-Schwarz inequality and the inequality
$\|\Delta_{J^{c}}\|_{2,1}\leqslant 3 \|\Delta_{J}\|_{2,1}$ we obtain
\begin{eqnarray*}
\frac{\left|\Delta_{J^{c}}\trans\Psi\Delta_{J}\right|}{\|\Delta_{J}\|^{2}}
&\leqslant& \frac{1}{7 \alpha s} \frac{\sum_{t=1}^{T}\sum_{j\in
J}\sum_{k\in J^c}|\Delta_{tj}||\Delta_{tk}|}
{\|\Delta_{J}\|^{2}}\\
&\leqslant&
\frac{1}{7\alpha s}\frac{\sum_{j\in J,k\in J^{c}}\|\Delta^{j}\|\|\Delta^{k}\|}{\|\Delta_{J}\|^{2}}\nonumber\\
&\leqslant& \frac{3}{7\alpha
s}\frac{\|\Delta_{J}\|_{2,1}^{2}}{\|\Delta_{J}\|^{2}}\\
&\leqslant &\frac{3}{7\alpha }\,.
\end{eqnarray*}
Combining these inequalities we find
\begin{eqnarray*}
\frac{\Delta \trans \Psi \Delta}{\|\Delta_J\|^{2}} \geqslant
\frac{\Delta_{J}\trans\Psi\Delta_{J}}{\|\Delta_{J}\|^{2}} +
\frac{2\Delta_{J^{c}}\trans\Psi\Delta_{J}}{\|\Delta_{J}\|^{2}}
\geqslant 1-\frac{1}{\alpha}>0.
\end{eqnarray*}
\end{proof}

Note also that, by an argument as in \cite{lounici2008snc}, it is
not hard to show that under Assumption \ref{mutcoh} the vector
$\beta^*$ satisfying (\ref{eq:model}) is unique.

Theorem \ref{th1} provides bounds for compound measures of risk, that
is, depending simultaneously on all the vectors $\beta^j$. An
important question is to evaluate the performance of estimators for
each of the components $\beta^j$ separately. The next theorem provides
a bound of this type and, as a consequence, a result on the selection
of sparsity pattern.

\begin{theorem}\label{est-supnorm}
Consider the model (\ref{eq:s1}) for $M\geqslant 2$ and
$T,n\geqslant 1$. Let the assumptions of Lemma \ref{lem:1} be
satisfied and let Assumption \ref{mutcoh} hold with the same $s$.
Set
$$
c = \left(3 + \frac{32}{7(\alpha-1)} \right)\sigma.
$$
Let $\lambda$, $A$ and $W_1,\dots,W_T$ be as in Lemma \ref{lem:1}.
Then with probability at least $1 - M^{1-q}$, where $q=\min(8\log M,
A\sqrt{T}/8)$, for any solution $\hat{\beta}$ of problem
(\ref{eq:opt}) we have
\begin{equation}\label{CI}
\frac{1}{\sqrt{T}}\|\hat{\beta}-\beta^{*}\|_{2,\infty} \leqslant
\frac{c}{\sqrt{n}}\sqrt{1 + \frac{A\log M}{\sqrt{T}}}.
\end{equation}
If, in addition,
\begin{equation}\label{lb}
\min_{j\in J(\beta^{*})}\frac{1}{\sqrt{T}} \|(\beta^{*})^j\| >
\frac{2c}{\sqrt{n}}\sqrt{1 + \frac{A\log M}{\sqrt{T}}},
\end{equation}
then with the same probability for any solution $\hat{\beta}$ of
problem (\ref{eq:opt}) the set of indices
\begin{equation}\label{lb1}
\hat{J} = \left\{j: \frac{1}{\sqrt{T}}\|\hat{\beta}^j\| >
\frac{c}{\sqrt{n}}\sqrt{1 + \frac{A\log M}{\sqrt{T}}}\right\}
\end{equation}
estimates correctly
the sparsity pattern $J(\beta^{*})$, that is,
$$
\hat{J} = J(\beta^{*}).
$$
\end{theorem}

\begin{proof}
Set $\Delta = \hat{\beta}-\beta^{*}$. Using Assumption \ref{mutcoh}
we obtain
\begin{align*}
\|\Delta\|_{2,\infty} &\leqslant \|\Psi\Delta\|_{2,\infty} +
\|(\Psi-I_{MT\times MT})\Delta\|_{2,\infty}\\
&\leqslant \| \Psi\Delta\|_{2,\infty}\\
&~~~+
 \max_{1\leqslant j \leqslant M}\left(\sum_{t=1}^{T}\left|\sum_{k=1,k \neq j}^{M}\Psi_{tj,tk}\Delta_{tk} \right|\right) \\
 &\leqslant \|\Psi\Delta\|_{2,\infty} \\
&~~~+ \sum_{k=1,k \neq j}^{M} \|\Delta^k\| \left(
\sum_{t=1}^{T}\max_{j \neq k}|\Psi_{tj,tk}|^{2} \right)^{1/2}\\
&\le \|\Psi\Delta\|_{2,\infty} +
\frac{\|\Delta\|_{2,1}\sqrt{T}}{7\alpha s}\,.
\end{align*}
Thus, by Lemma \ref{lem:1} and Theorem 3.1, with probability at
least $1 - M^{1-q}$,
\begin{equation*}
\|\Delta\|_{2,\infty} \leqslant \left(\frac{3}{2} +\frac{16}{7
\alpha \kappa^{2}}\right)\lambda T.
\end{equation*}
By Lemma \ref{lem:2}, $\alpha \kappa^{2}=\alpha-1$, which yields the
first result of the theorem. The second result follows from the
first one in an obvious way.
\end{proof}

Assumption of type (\ref{lb}) is inevitable in the context of
selection of sparsity pattern. It says that the vectors
$(\beta^*)^j$ cannot be arbitrarily close to 0 for $j$ in the
pattern. Their norms should be at least somewhat larger than the
noise level.

The second result of Theorem \ref{est-supnorm} (selection of
sparsity pattern) can be compared with \cite{bach07, NarRin08} who
considered the Group Lasso. There are several differences. First,
our estimator $\hat J$ is based on thresholding of the norms
$\|\hat\beta^j\|$, while \cite{bach07, NarRin08} take instead the
set where these norms do not vanish. In practice, the latter is known to be
a poor selector; it typically overestimates the true sparsity
pattern. Second, \cite{bach07, NarRin08} consider specific
asymptotic settings, while our result holds for any fixed $n,M,T$.
Different kinds of asymptotics can be therefore obtained as simple
corollaries. Finally, note that the estimator $\hat\beta$ is not
necessarily unique. Though \cite{NarRin08} does not discuss this
fact, the proof there only shows that {\it there exists a
subsequence of solutions $\hat{\beta}$ of (\ref{eq:opt})} such that
the set $\{j: \|\hat{\beta}^j\| \ne 0\}$ coincides with the sparsity
pattern $J(\beta^{*})$ in some specified asymptotics (we note here
the ``if and only if" claim before formula (23) in \cite{NarRin08} is
not proved). In contrast, the argument in Theorem \ref{est-supnorm}
does not require any analysis of the uniqueness issues, though it is
not excluded that the solution is indeed unique. It guarantees that
{\it simultaneously for all solutions $\hat{\beta}$ of
(\ref{eq:opt})} and any fixed $n,M,T$ the correct selection is done
with high probability.

Theorems 3.3 and \ref{est-supnorm} imply the following corollary.

\begin{corollary}\label{p-norm}
Consider the model (\ref{eq:s1}) for $M\geqslant 2$ and
$T,n\geqslant 1$. Let the assumptions of Lemma \ref{lem:1} be
satisfied and let Assumption \ref{mutcoh} holds with the same $s$.
Let $\lambda$, $A$ and $W_1,\dots,W_T$ be as in Lemma \ref{lem:1}.
Then with probability at least $1 - M^{1-q}$, where $q=\min(8\log M,
A\sqrt{T}/8)$, for any solution $\hat{\beta}$ of problem
(\ref{eq:opt}) and any $1\le p <\infty$ we have
\begin{equation}
\frac{1}{\sqrt{T}}\|\hat{\beta}-\beta^{*}\|_{2,p} \leqslant c_1
\sigma  \frac{s^{1/p}}{\sqrt{n}}\sqrt{1 + \frac{A\log
M}{\sqrt{T}}}\,,
\end{equation}
where
$$
c_{1} = \left( \frac{32\alpha}{\alpha-1} \right)^{1/p}\left(3+
\frac{32}{7(\alpha-1)} \right)^{1-\frac{1}{p}} .$$
If, in addition,
(\ref{lb}) holds, then with the same probability for any solution
$\hat{\beta}$ of problem (\ref{eq:opt}) and any $1\le p <\infty$ we
have
\begin{equation}\label{CI1}
\frac{1}{\sqrt{T}}\|\hat{\beta}-\beta^{*}\|_{2,p} \leqslant c_1
\sigma  \frac{|\hat J|^{1/p}}{\sqrt{n}}\sqrt{1 + \frac{A\log
M}{\sqrt{T}}}\,,
\end{equation}
where $\hat J$ is defined in (\ref{lb1}).
\end{corollary}

\begin{proof}
Set $\Delta = \hat{\beta} - \beta$. For any $p\geqslant 1$ we have
\begin{equation*}
\frac{1}{\sqrt{T}}\| \Delta \|_{2,p} \leqslant \left(
 \frac{1}{\sqrt{T}}\| \Delta \|_{2,1} \right)^{\frac{1}{p}}\left(
 \frac{1}{\sqrt{T}}\| \Delta \|_{2,\infty} \right)^{1-\frac{1}{p}}.
\end{equation*}
Combining (\ref{eq:t2}), (\ref{CI}) with $\kappa =
\sqrt{1-\frac{1}{\alpha}}$ and the above display yields the first
result.
\end{proof}

Inequalities (\ref{CI}) and (\ref{CI1}) provide confidence intervals
for the unknown parameter $\beta^*$ in mixed (2,$p$)-norms.

For averages of the coefficients $\beta_{tj}$ we can establish a
sign consistency result which is somewhat stronger than the result in Theorem
\ref{est-supnorm}. For any $\beta\in\R^{M}$, define
$\vec{\mathrm{sign}}(\beta)=
(\mathrm{sign}(\beta^{1}),\ldots,\mathrm{sign}(\beta^{M}))^{\trans}$
where
\begin{equation*}
\mathrm{sign}(t)=\begin{cases}
1 &\text{if $t>0$},\\
0 &\text{if $t=0$},\\
-1 &\text{if $t<0$}.
\end{cases}
\end{equation*}
Introduce the averages
$$ {a}_j^* = \frac{1}{T} \sum_{t=1}^T \beta_{tj}^*, \quad \hat{a}_j = \frac{1}{T} \sum_{t=1}^T \hbeta_{tj}.
$$
Consider the threshold $\tau = \frac{c}{\sqrt{n}}\sqrt{1 +
\frac{A\log M}{\sqrt{T}}}$ and define a thresholded estimator
$$
\tilde{a}_j = \hat{a}_j I \big\{|\hat{a}_j| > \tau\big\}.
$$
Let $\tilde{a}$ and ${a}^*$ be the vectors with components
$\tilde{a}_j$ and ${a}_j^*$, $j=1,\dots,M$, respectively. We need
the following additional assumption.
\begin{assumption}\label{ass:2} It holds:
$$
\min_{j \in J(a^*)} |a_j^*| \geq \frac{2c}{\sqrt{n}}\sqrt{1 +
\frac{A\log M}{\sqrt{T}}}.
$$
\end{assumption}
This assumption says that we cannot recover arbitrarily small
components. Similar assumptions are standard in the literature on
sign consistency (see, for example, \cite{lounici2008snc} for more details
and references).

\begin{theorem}\label{Signnonasymp}
Consider the model (\ref{eq:s1}) for $M\geqslant 2$ and
$T,n\geqslant 1$. Let the assumptions of Lemma \ref{lem:1} be
satisfied and let Assumption \ref{mutcoh} hold with the same $s$.
Let $\lambda$ and $A$ be defined as in Lemma \ref{lem:1} and $c$ as
in Theorem \ref{est-supnorm}. Then with probability at least $1 -
M^{1-q}$, where $q=\min(8\log M, A\sqrt{T}/8)$, for any solution
$\hat{\beta}$ of problem (\ref{eq:opt}) we have
\begin{equation*}
\max_{1\leqslant j \leqslant M}|\hat{a}_{j} - a_{j}^*| \leqslant
\frac{c}{\sqrt{n}}\sqrt{1 + \frac{A\log M}{\sqrt{T}}}.
\end{equation*}
If, in addition, Assumption \ref{ass:2} holds, then with the same
probability, for any solution $\hat{\beta}$ of problem
(\ref{eq:opt}), $\tilde{a}$ recovers the sign pattern of ${a}^*$:
\begin{equation*}
\vec{\mathrm{sign}}(\tilde{a})=\vec{\mathrm{sign}}(a^*).
\end{equation*}
\end{theorem}
\begin{proof}
Note that for every $j \in \N_M$
$$
|\hat{a}_j - a_j^*| \leq \frac{1}{\sqrt{T}} \|\hbeta-
\beta^*\|_{2,\infty} \leq \frac{c}{\sqrt{n}}\sqrt{1 + \frac{A\log
M}{\sqrt{T}}}.
$$
The proof is then similar to that of Theorem \ref{est-supnorm}.
\end{proof}

We may consider a stronger assumption that $\beta^*_{t} = a$ for
every $t \in \N_T$,  where $a = (a_j : j \in \N_M)\in \R^M$ is an
unknown vector to be estimated. Then Theorem \ref{Signnonasymp}
implies that $\hat{a}$ is a $\sqrt{n}$-consistent (up to logarithms)
estimator of all the components of $a$ and the sparsity (and sign)
pattern of $a$ is correctly recovered by that of $\tilde{a}$ with
overwhelming probability.

\section{Non-Gaussian noise}
\label{sec:5}
In this section, we only assume that the random variables $W_{ti},i
\in \N_{n}, t \in \N_{T}$, are independent with zero mean and finite
variance $\E[W_{ti}^{2}]\leqslant \sigma^{2}$. In this case the
results remain similar to those of the previous sections, though the
concentration effect is weaker. We need the following technical
assumption
\begin{assumption}\label{tech-ass}
The matrix $X$ is such that
\begin{equation*}
\frac{1}{nT}\sum_{t=1}^{T}\sum_{i=1}^{n}\max_{1\leqslant j
\leqslant M}|(x_{ti})_{j}|^{2}\leqslant c',\end{equation*} for a
constant $c'>0$.
\end{assumption}
This assumption is quite mild. It is satisfied for example, if
all $(x_{ti})_{j}$ are bounded in absolute value by a constant
uniformly in $i,t,j$. We have the two following theorems.
\begin{theorem}\label{pred-est-RE}
Consider the model \eqref{eq:s1} for $M \geq 3$ and $T,n \geq 1$.
Assume that the random vectors $W_1,\dots,W_T$ are independent
with zero mean and finite variance $\E[W_{ti}^{2}]\leqslant
\sigma^{2}$, all diagonal elements of the matrix $X\trans X/n$
are equal to $1$ and $M(\beta^*) \leq s$. Let also Assumption
\ref{tech-ass} be satisfied. Furthermore let $\kappa$ be defined
as in Assumption \ref{ass} and $\phi_{\rm max}$ be the largest
eigenvalue of the matrix $X\trans X/n$. Let
$$
\lambda = \sigma \sqrt{\frac{(\log M)^{1+\delta}}{nT}},~~~\delta
> 0.
$$
Then with probability at least $1 - \frac{(2e\log M - e)c'}{(\log
 M)^{1+\delta}}$, for
any solution $\hbeta$ of problem \eqref{eq:opt} we have
\begin{eqnarray}
\label{eq:t11} \hspace{-.2truecm}\frac{1}{nT} \|X(\hbeta -
\beta^*)\|^2 \hspace{-.2truecm} & \leq &
\hspace{-.2truecm} \frac{16}{\kappa^2} \sigma^2 s \frac{(\log M)^{1+\delta}}{n},\hspace{.3truecm} \\
\label{eq:t21} \frac1{\sqrt {T}}\|\hbeta - \beta^*\|_{2,1}
\hspace{-.2truecm} & \leq & \hspace{-.2truecm} \frac{16}{\kappa^2}
\sigma s \sqrt{\frac{(\log M)^{1+\delta}}{n}}, \\
\label{eq:t31} \hspace{-.2truecm} M(\hbeta) \hspace{-.2truecm} &
\leq & \hspace{-.2truecm}  \frac{64 \phi_{\rm max}}{\kappa^2} s.
\end{eqnarray}
If, in addition, Assumption RE(2$s$) holds, then with the same
probability for any solution $\hbeta$ of problem \eqref{eq:opt} we
have
\begin{eqnarray}
\nonumber
\hspace{-.2truecm}\frac{1}{T} \|\hbeta - \beta^*\|^2
\hspace{-.2truecm} & \leq & \hspace{-.2truecm} \frac{160
}{\kappa^4(2s)} \sigma^2 s \frac{(\log M)^{1+\delta}}{n}\,.
\hspace{.3truecm}
\end{eqnarray}
\end{theorem}

\begin{theorem}
Consider the model (\ref{eq:s1}) for $M\geqslant 3$ and
$T,n\geqslant 1$. Let the assumptions of Theorem \ref{pred-est-RE}
be satisfied and let Assumption \ref{mutcoh} hold with the same
$s$. Set
$$
c = \left(\frac{3}{2} + \frac{1}{7(\alpha-1)} \right)\sigma.
$$
Let $\lambda$ be as in Theorem as in \ref{pred-est-RE}. Then with
probability at least $1 - \frac{(2e\log M - e)c'}{(\log
(MT))^{1+\delta}}$, for any solution $\hat{\beta}$ of problem
(\ref{eq:opt}) we have
\begin{equation*}
\frac{1}{\sqrt{T}}\|\hat{\beta}-\beta^{*}\|_{2,\infty} \leqslant c
\sqrt{\frac{(\log  M)^{1+\delta}}{n }}.
\end{equation*}
If, in addition, it holds that
$$\min_{j\in J(\beta^{*})}\frac{1}{\sqrt{T}}
\|(\beta^{*})^j\| > 2c \sqrt{\frac{(\log M)^{1+\delta}}{n }},
$$
then with the same probability for any solution $\hat{\beta}$ of
problem (\ref{eq:opt}) the set of indices
$$
\hat{J} = \Big\{j:
\frac{1}{\sqrt{T}}\|\hat{\beta}^j\| > c \sqrt{\frac{(\log
M)^{1+\delta}}{n }}\Big\}
$$ estimates correctly the sparsity
pattern $J(\beta^{*})$:
$$
\hat{J} = J(\beta^{*}).
$$
\end{theorem}

The proofs of these theorems are similar to the ones of Theorems
\ref{th1} and \ref{est-supnorm} up to a modification of the bound on
$P(\mathcal{A}^{c})$ in Lemma \ref{lem:1}. We consider now the event
\begin{equation*}
\mathcal{A}=\left\{\max_{j=1}^M \sqrt{\sum_{t=1}^T
\left(\sum_{i=1}^n (x_{ti})_j W_{ti} \right)^2}\leq \lambda nT
\right\}.
\end{equation*}
The Markov inequality yields that
\begin{equation*}
{\rm Pr}(\mathcal{A}^{c})\leqslant
\frac{\sum_{t=1}^{T}\E[\max_{1\leqslant j \leqslant
M}\left(\sum_{i=1}^n (x_{ti})_j W_{ti} \right)^{2}]}{(\lambda
nT)^{2}}.
\end{equation*}
Then we use Lemma \ref{nem} given below with the random vectors
$$Y_{ti}=((x_{ti})_{1}W_{ti}/n,\ldots,(x_{ti})_{M}W_{ti}/n)\in\R^{M},$$
$\forall i\in \N_{n}$, $\forall t \in \N_{T}$. We get that
\begin{eqnarray*}
{\rm Pr}(\mathcal{A}^{c})&\leqslant& \frac{2e\log M -
e}{\lambda^{2}nT}\sigma^{2}\frac{1}{nT}\sum_{t=1}^{T}
\sum_{i=1}^{n}\max_{1\leqslant j \leqslant M}|(x_{ti})_{j}|^{2}.\\
\end{eqnarray*}
By the definition of $\lambda$ in Theorem 5.2 and Assumption
\ref{tech-ass} we obtain
\begin{eqnarray*}
{\rm Pr}(\mathcal{A}^{c})&\leqslant& \frac{(2e\log M - e)c'}{(\log
 M)^{1+\delta}}.
\end{eqnarray*}
$\qed$

Thus, we see that under the finite variance assumption on the
noise, the dependence on the dimension $M$ cannot be made negligible
for large $T$.

\appendix
\section{Auxiliary results}

Here we collect two auxiliary results which are used in the above
analysis.  The first result is a useful bound on the tail of the
chi-square distribution.

\begin{lemma}\label{chi}
Let $\chi_T^2$ be a chi-square random variable with $T$ degrees of
freedom. Then
\begin{equation*}
{\rm Pr}(\chi_T^2 > T +x) \le
\exp\left(-\frac{1}{8}\min\left(x,\frac{x^2}{T}\right)\right)
\end{equation*}
for all $x>0$.
\end{lemma}
\begin{proof} By the Wallace inequality \cite{Wa59} we have
$${\rm Pr}(\chi_T^2 > T +x) \le  {\rm Pr}({\mathcal N} > z(x)),$$
where ${\mathcal N}$ is the standard normal random variable and
$z(x)=\sqrt{x-T\log(1+x/T)}$. The result now follows from
inequalities ${\rm Pr}({\mathcal N} > z(x))\le \exp(-z^2(x)/2)$
and $$u-\log(1+u)\ge \frac{u^2}{2(1+u)}\ge \frac{1}{4}\min\left(u,
u^2\right), \ \forall u>0.$$
\end{proof}

The next result is a version of Nemirovski's inequality (see
\cite{dumbgen2008nir}, Corollary 2.4 page 5).
\begin{lemma}\label{nem}
Let $Y_{1},\ldots,Y_{n}\in\R^{M}$ be independent random vectors
with zero means and finite variance, and let $M\geqslant 3$. Then
\begin{equation*}
\E\left[|\sum_{i=1}^{n}Y_{i}|_{\infty}^{2}\right]\leqslant (2e\log M
- e) \sum_{i=1}^{n}\E\left[|Y_{i}|_{\infty}^{2}\right],
\end{equation*}
where $|\cdot|_{\infty}$ is the $\ell_\infty$ norm.
\end{lemma}

\bibliography{panel}

\end{document}